\documentclass[letterpaper, 10 pt, conference]{ieeeconf}  

\IEEEoverridecommandlockouts                              

\usepackage{graphicx}

\usepackage{soul,xcolor}

\setstcolor{red}

\usepackage{xcolor}

\usepackage{amsmath}

\usepackage{amsfonts}

\usepackage{wrapfig}

\usepackage[linesnumbered,ruled,vlined]{algorithm2e}

\usepackage{caption}
\newcommand{\psdnew}[1]{\textcolor{black}{#1}} 

\newcommand{\tnm}[1]{{\textbf{\texttt{#1}}}} 

\SetCommentSty{mycommfont}

\SetKwInput{KwInput}{Input}                
\SetKwInput{KwOutput}{Output}              

\usepackage{times}

\usepackage{multicol}
\usepackage[bookmarks=true]{hyperref}
\usepackage{color}

\usepackage{dsfont}
\usepackage{microtype}
\usepackage{graphicx}
\usepackage{subfig}
\usepackage{booktabs} 
\usepackage{xspace}

\usepackage{footnote}
\makesavenoteenv{tabular}
\makesavenoteenv{table}

\usepackage{epsfig}
\usepackage{amssymb}
\usepackage{amsmath}
\usepackage{amsfonts}

\usepackage{stackengine}
\usepackage{algorithmic,comment}

\usepackage{colortbl}
\usepackage{xcolor}
\definecolor{LightCyan}{rgb}{0.88,1,1}

\usepackage{color}
\newcommand{\SC}[1]{{\color{black}#1}}

\newcommand{\tu}[1]{{\underline{\textit{#1}}}}

\newcommand{\expec}{\mathbb{E}}
\newcommand{\reals}{\mathbb{R}}
\DeclareMathOperator{\indicator}{\mathds{1}}

\newcommand{\yoracle}{y_{\mathrm{oracle}}}

\newcommand{\pioffload}{\pi}

\newcommand{\Roffload}{R}
\newcommand{\aoffload}{a}

\newcommand{\costcloud}{c_{\mathrm{slow}}}
\newcommand{\costrobot}{c_{\mathrm{fast}}}
\newcommand{\alphaaccuracy}{\alpha}

\newcommand{\betacost}{\beta}

\newcommand{\DeltaL}{\Delta \mathcal{L}}

\newcommand{\argmax}{\mathop{\rm argmax}}

\newtheorem{thm}{Theorem}

\newtheorem{lem}{Lemma}

\newtheorem{problem}{Problem}
\newtheorem{assumption}{Assumption}
\newtheorem{property}{Property}
\newtheorem{example}{Example}

\newcommand{\frobot}{f_\mathrm{fast}}

\newcommand{\fcloud}{f_\mathrm{slow}}

\pdfoutput=1

\renewcommand{\baselinestretch}{0.995}

\title{\LARGE \bf
Interpretable Trade-offs Between Robot Task Accuracy and Compute Efficiency
}

\author{Bineet Ghosh$^{1}$, Sandeep Chinchali$^{2}$, Parasara Sridhar Duggirala$^{1}$
\thanks{$^{1}$Bineet Ghosh and Parasara Sridhar Duggirala are with the Department of Computer Science,
        The University of North Carolina at Chapel Hill, USA
    {\tt\small \{bineet,psd\}@cs.unc.edu}}%
\thanks{$^{2}$Sandeep Chinchali is with the Department of Electrical and Computer Engineering, The University of Texas at Austin, USA
        {\tt\small sandeepc@utexas.edu}}%
}

\begin{document}

\maketitle
\thispagestyle{empty}
\pagestyle{empty}

\begin{abstract}

%
%
A robot can invoke heterogeneous computation resources such as CPUs, cloud GPU servers, or even human computation for achieving a high-level goal.
The problem of invoking an appropriate computation model so that it will successfully complete a task while keeping its compute and energy costs within a budget is called a \emph{model selection problem}.
In this paper, we present an optimal solution to the model selection problem with two compute models, the first being fast but less accurate, and the second being slow but more accurate.
The main insight behind our solution is that \emph{a robot should invoke the slower compute model only when the benefits from the gain in accuracy outweigh the computational costs.}
We show that such cost-benefit analysis can be performed by leveraging the statistical correlation between the accuracy of fast and slow compute models.
\psdnew{We demonstrate the broad applicability of our approach to diverse problems such as perception using neural networks and safe navigation of a simulated Mars rover.} 


\end{abstract}


\section{Introduction}
\label{sec:intro}
Ideally, robotic computation should be highly accurate, responsive, and fast, as well as compute-and-power-efficient. 
Modern robots, however, face the challenge of selecting from an array of heterogeneous compute resources, each with a unique trade-off between accuracy and compute cost.
For example, should a factory robot trust the perception results from 
\psdnew{an on-board deep neural network (DNN)}
or ask a busy human supervisor for help? Likewise, should a small drone compute
its motion plan locally, or wait for a higher-fidelity 
\psdnew{plan} from a remote server?
At their core, these scenarios are 
\psdnew{instances of a} \textit{compute model selection} \psdnew{problem}, where a robot must gracefully balance task-relevant accuracy with compute time, power, or network and human-processing delay. 

Figure~\ref{fig:one} illustrates the model selection problem addressed in this paper. 
Given the sensor observations $x$ at each time step, a robot's model selection policy 
$\pi$ must dynamically invoke either a fast, compute-and-power-efficient model ($f_{\mathrm{fast}}$) or a slower, more accurate model ($f_{\mathrm{slow}}$) based on a high-level task's required accuracy. 
Variants of this problem have been studied for perception tasks in cloud robotics \cite{chinchali2019network,rahman2017motion} and human-robot collaboration \cite{cakmak2012designing,whitney2017reducing,kaipa2016enhancing}. 
However, existing works either offer specialized point-solutions (e.g. for perception \cite{eshratifar2020runtime,dorka2020modality}) that do not readily generalize to other domains, use hand-engineered heuristics, or employ uninterpretable, learning-based algorithms \cite{chinchali2019network,dinh2018learning}. 
\psdnew{\emph{Our key contribution is to provide a unified, interpretable, and theoretically-grounded framework for compute model selection in robotics.}} 

The fundamental principle behind selecting an appropriate compute model is to perform a cost-benefit analysis.
%
Our key insight is that a robot's model selection algorithm can leverage the statistical, and often analytical, correlation between the accuracy of the fast and the slow compute models.
This correlation can enable us to perform \emph{reliable and interpretable} cost-benefit analysis between compute cost and gain in accuracy for the different models.
%
%
Crucially, such correlations between fast and slow models are now possible even for state-of-the-art DNNs, due to recent advances that compress large DNNs with provable approximation guarantees \cite{baykal2019datadependent,liebenwein2020provable}.

\begin{figure}[t]
\vskip 0.22in
\begin{center}
{\includegraphics[width=1.0\columnwidth]{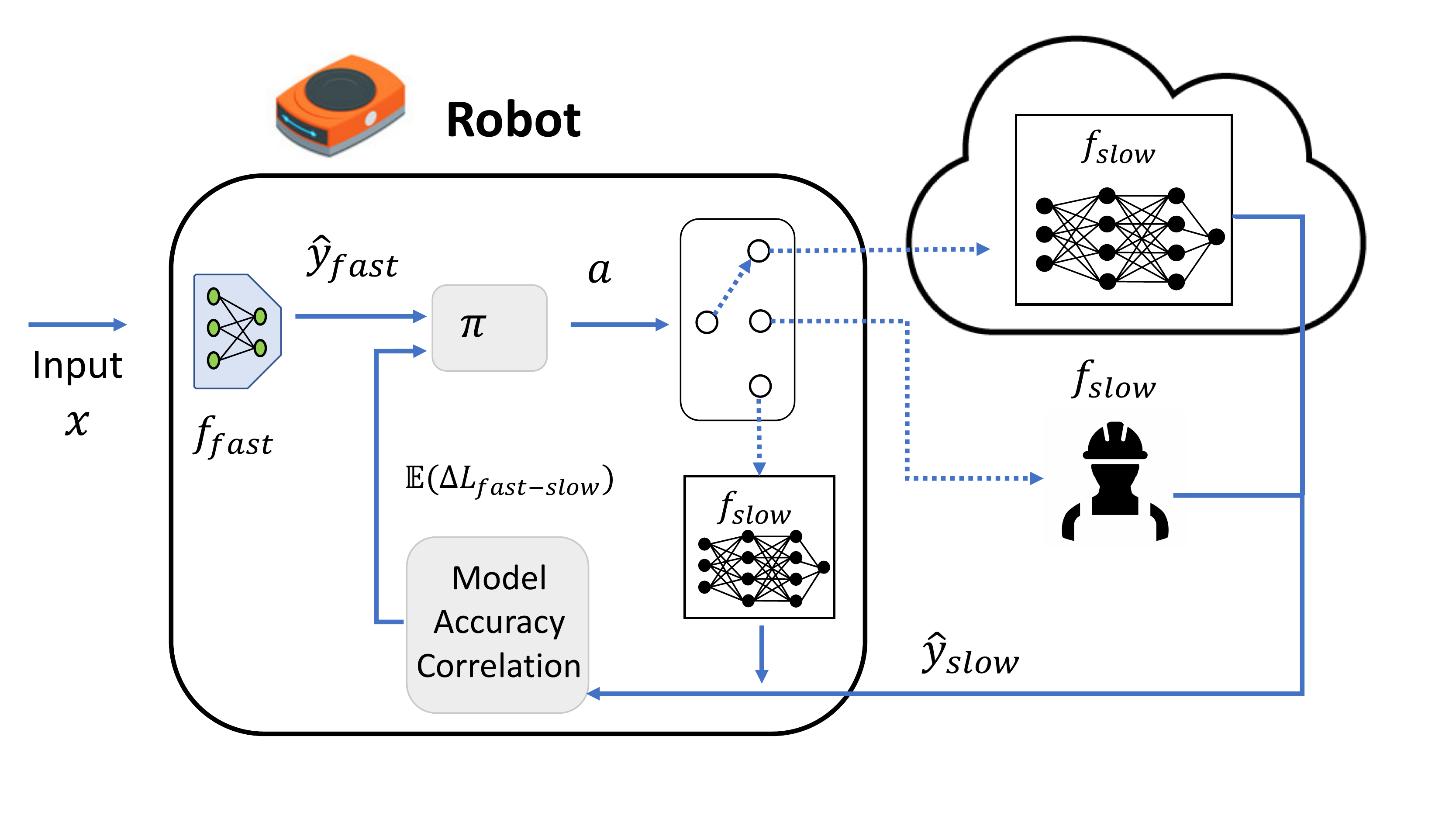}}
\end{center}
\vskip -0.2in
\caption{\textbf{The Compute Model Selection Problem: }
A robot must balance task accuracy and compute cost, such as energy or latency, when choosing between heterogeneous compute resources. Our interpretable model selection policy $\pi$ leverages the statistical correlation between fast and slow compute models $f_{\mathrm{fast}}$ and $f_{\mathrm{slow}}$ to dynamically decide which model to invoke.}
\label{fig:one}
\vskip -0.25in
\end{figure}

\textit{Literature Review: }
Our work is broadly related to computational offloading in cloud robotics as well as teacher feedback for human-robot interaction. 
The closest work to ours is \cite{chinchali2019network}, which develops a deep reinforcement learning (RL) policy to select between a fast, less accurate deep neural network (DNN) or slower, more accurate DNN running at a cloud computing server. Indeed, we explicitly build upon \cite{chinchali2019network} in our formulation by considering fast and slow compute models with a hierarchy of compute costs. In stark contrast to \cite{chinchali2019network}, however, we avoid uninterpretable, RL-based model selection \psdnew{policies. 
Instead we leverage statistical correlations between fast and slow computational models, such as compression algorithms for DNNs \cite{baykal2019datadependent,liebenwein2020provable}.
}
%
Further, unlike \cite{chinchali2019network}, we introduce a theoretically-grounded cost-benefit analysis for model selection, which generalizes beyond DNNs to high-dimensional linear regression and even sampling-based reachability problems, as shown in our evaluation. 

Our work is also inspired by methods to compress large DNNs for efficient inference on compute-and-power limited robots. For example, the EfficientNet \cite{tan2019efficientnet} suite of vision models provides 7 model variants that trade-off accuracy with model size and latency. Importantly, recent methods utilize core-set theory to train fast, compressed DNNs that provably approximate a slower DNN by pruning convolutional filters based on sensitivity analyses \cite{baykal2019datadependent,liebenwein2020provable}.
%
 

Finally, our work is related to scenarios where a robot must selectively ask a human teacher for clarification during active learning tasks \cite{cakmak2012designing,whitney2017reducing} or remote assistance for manipulation \cite{kaipa2016enhancing}. In principle, our framework applies to such settings if a robot can accurately correlate its confidence with the marginal accuracy gain it receives from human feedback. In practice, however, such correlations can often be learned from historical interaction data but are hard to \textit{analytically} quantify.

\textit{Contributions: } Given prior literature, our contributions are three-fold. First, we design an interpretable model selection algorithm which leverages analytical correlations between fast and slow model performance to dynamically decide which model to invoke. Second, we show how our algorithm can naturally leverage recent advances that compress large DNNs with provable approximation guarantees that relate fast and slow models. Third, we show strong experimental performance of our algorithm on diverse domains ranging from robotic perception to sampling-based reachability analysis for a simulated rover navigating Martian terrain data.
 
\textit{Organization: } This paper is organized as follows. In Section \ref{sec:probStatement}, we introduce a general formulation for model selection to gracefully trade-off task accuracy and compute costs. 
In Sec. \ref{sec:modelSelection}, we provide theoretical guarantees for model selection and instantiate them for applications in perception and reachability analysis in Sec. \ref{sec:app_scenarios}. 
Finally, we provide our experimental results in Sec. \ref{sec:eval} and conclude in Sec. \ref{sec:conclusion}.



\section{Problem Statement}
\label{sec:probStatement}
In this section, we formally define the problem of model selection depicted in Fig.~\ref{fig:one} by introducing compute models, an accuracy metric, and a performance criterion. 

\tu{Compute Model Input}:
The input to the compute model, at time $t$, is denoted by $x^t \in \mathbb{R}^n$. We denote the input data distribution by $\mathcal{X}$, that is, $x^t \sim  \mathcal{X}$. In practice, $x^t$ could represent a depth-camera image or laser scan.

\tu{Compute Models}:
The compute models are denoted by $f_i: \mathbb{R}^n \rightarrow \mathbb{R}^m$, where $i \in \{\text{fast}, \text{slow}\}$. Given an input $x^t$, the output is denoted by $y_i^t=f_i(x^t)$. The cost associated with $f_{i}$ is given by $c_{i} \in \mathbb{R}_{+}$. The cost is context-dependent, such as battery consumption, compute inference latency, or even communication latency for cloud robotics tasks. For example, the compute models could be a DNN, with image input $x^t$ and corresponding segmentation $y^t$. The distribution of outputs is denoted by $\mathcal{Y}$, that is, $y_i^t \sim  \mathcal{Y}$. 
The ground-truth output associated with input $x^t$ is denoted by $y^t_{\text{oracle}}$.


\tu{Loss Function}: Let $\mathcal{L}(y^t_1,y^t_2): \mathbb{R}^m \times \mathbb{R}^m \rightarrow \mathbb{R}_{\ge 0}$ be the loss function, that formally quantifies the quality of the output $y^t_1$ returned by a compute model compared to the ground-truth result of $y^t_2$. A lower value of $\mathcal{L}(\cdot)$ indicates a more accurate output. In practice, the loss function is context-dependent, such as the cross-entropy loss for image classification.

\tu{Model Selection Policy}: Given an input $x^t$, the model selection policy decides whether to use the slow compute model $f_{\text{slow}}$, or the fast model $f_{\text{fast}}$. 
We assume that the policy has access to the results of the fast model (without this information, the policy would be purely random).
%
Therefore, the problem of model selection is to infer whether or not to \textit{additionally} invoke the slow model to enhance task accuracy if the fast model results are insufficient for a robot's high-level goal.
The challenge is that the robot only has access to the input $x^t$ and the fast model output $y^t_{\text{fast}}$ and thus must estimate the accuracy benefit of the slow model \textit{before} invoking it. 

Formally, we define the model selection policy as $\pioffload: \mathbb{R}^n \times \mathbb{R}^m \rightarrow \{0, 1\}$. Given input $x^t$ and fast model prediction $y^t_{\text{fast}}$, we define the action as $\aoffload^t=\pioffload(x^t,y^t_{\text{fast}})$. The action $\aoffload^t=0$ indicates selecting the fast model $f_\text{fast}$, and $\aoffload^t=1$ indicates selecting the slow model $f_\text{slow}$. We define the cost associated with each action as $\mathtt{cost}(\aoffload^t)$:
\begin{equation*}
    \mathtt{cost}(\aoffload^t) = \begin{cases} c_{\text{fast}} &\mbox{if } \aoffload^t = 0 \\
    c_{\text{fast}} + c_{\text{slow}} &\mbox{if } \aoffload^t = 1 \end{cases}.
\end{equation*}

\tu{Reward}: To simultaneously achieve high task accuracy while minimizing the cost of compute, we introduce a per-timestep reward. Given input $x^t$, the output of the fast model $y^t_{\text{fast}}$, and model selection $\aoffload^t$, the corresponding reward is: 
\begin{equation}
    \Roffload^t(\aoffload^t) = \begin{cases} -\alphaaccuracy \mathcal{L}(y^t_{\text{fast}},y^t_{\text{oracle}}) 
    - \betacost \mathtt{cost}(0) ~~~\mbox{if } \aoffload^t = 0 \\
    -\alphaaccuracy  \mathcal{L}(y^t_{\text{slow}},y^t_{\text{oracle}})
    - \betacost \mathtt{cost}(1) ~~~ \mbox{if } \aoffload^t = 1 \end{cases}
\label{eq:inference_reward}
\end{equation}
where $\alphaaccuracy$, $\betacost$ $\in \mathbb{R}_{+}$ are user-defined weights to balance
the emphasis on accuracy and cost. These can be flexibly set by a roboticist given the unique requirements of a high-level task. For example, a fleet of low-power, compute-limited warehouse robots that rarely interact with humans might have a higher emphasis $\betacost$ on cost to minimize how many times they query a shared central server or remote human supervisor. Conversely, robots that operate in safety-critical scenarios will have a much higher emphasis on accuracy given by $\alpha$.

\subsection{Formal Problem Definition}
Given a stream of $N$ inputs,
$\{x^1, x^2, \cdots, x^N\}$, our goal is to propose an optimal model selection policy $\pioffload^*$, that provably maximizes the expected cumulative reward:
$$
\expec{\sum_{t=0}^{N} \Roffload^t \big(\pioffload^*(x^t,y^t_{\text{fast}})\big)}
$$
Intuitively, $\pioffload^*$ achieves the optimal balance between the cost and accuracy over the given period of $N$ time steps. We now formally define the model selection problem.

\begin{problem}[Model Selection for Inference]
Given fast model $\frobot$, slow model $\fcloud$, loss function $\mathcal{L}(\cdot)$, and model selection cost $\mathtt{cost}(\cdot)$, find the optimal model selection policy $\pioffload^{*}$, which maximizes the reward (Equation \ref{eq:inference_reward}) over a finite horizon $N$:
    \begin{align*}
        \pioffload^{*} = \argmax_{\pioffload} \expec{\sum_{t=0}^{N} \Roffload^t \big(a^t = \pioffload(x^t, y^t_{\text{fast}}) \big)}. 
    \end{align*}
\label{prob:inference}
\end{problem}
Section \ref{sec:modelSelection} provides our solution to Problem \ref{prob:inference}.

\subsection{Discussion on the Problem Definition}
\label{subsec:discussionProb}
The model selection problem is broadly applicable in robotics since it is agnostic to the nature of the compute models, the loss function, or even costs. For example, the models could represent small quantized and large, compute-intensive DNNs or even small and large databases or random forests. Further, the costs could represent battery consumption or communication delay or model inference time. 

The main challenge of this problem is the limited information available to the selection policy, namely the input $x^t$, fast model output $y^t_{\text{fast}}$, and the cost function $\mathtt{cost}(\cdot)$. The key challenge is to estimate the accuracy of the slow model, $f_{\text{slow}}$, before even invoking it, which motivates our key technical approach to statistically relate both models' accuracy.


\section{An Algorithmic Approach To Model Selection}
\label{sec:modelSelection}
In this section, we provide an optimal solution to the model selection problem (Problem \ref{prob:inference}). First, we make the following practically-motivated assumption.

\begin{assumption}[Action and State Independence]
\label{assumption:probOff}
Given a model input $x^t$ at any time $t$, the model selection $a^t$ of policy $\pioffload$ does not affect the next robot measurement $x^{t+1}$. 
\end{assumption}

Our assumption is practical in many robotics scenarios, since $a^t$ is simply a choice of a compute model to process inputs, not a physical \textit{actuation} decision. For example, a robot can run a fast perception DNN on images $x^t$ at every timestep and its choice to optionally consult a slower DNN $a^t$ does not affect the new image observation $x^{t+1}$, which is instead largely affected by its ego-motion and surroundings. Our assumption will not hold for fast-moving robots whose control decisions are heavily dependent on the perception model they invoke, which we discuss in our future work.


\begin{thm}
    The optimal model selection policy that solves Problem \ref{prob:inference} is of the form:
\begin{align*}
    \pioffload^{*}(x^t,y^t_{\text{fast}}) = \indicator 
    \bigg (\frac{\betacost}{\alphaaccuracy} \costcloud &< \\ \expec \underbrace{\big( \mathcal{L}(y^t_{\text{fast}},\yoracle^t) - \mathcal{L}(y^t_{\text{slow}},\yoracle^t) \bigg)}_{\mathrm{~task~accuracy~gain}}
\end{align*}
\label{thm:gen_offload_oracle}
\end{thm}

\begin{proof}
    By Assumption \ref{assumption:probOff}, the action $a^t$ at every time does not affect the next state $x^{t+1}$. Thus, given any input $x^t$, the actions $a^t$ are independent, so to maximize the cumulative reward it suffices to maximize the reward at every time-step independently.
Recall from Equation \ref{eq:inference_reward} that the reward depends on two choices of $\aoffload^t$, that is, $\aoffload^t \in \{0,1\}$. Therefore, Problem \ref{prob:inference} can be rewritten as:
\begin{small}
\begin{align*}
    \pioffload^{*}(x^t, y^t_{\mathrm{fast}}) &= \argmax_{\aoffload^t \in \{0,1\}} \expec(\Roffload^t(\aoffload^t))
%
\end{align*}
\end{small}
Substituting in the reward definition (Equation \ref{eq:inference_reward}),
we see that we should choose the slow model only when the associated reward is higher than continuing with the fast model. Thus, we choose $\aoffload^t = 1$ only when:
    \begin{small}
    \begin{align*}
        -\alphaaccuracy \expec\big(\mathcal{L}(\yoracle^t, y_{\text{slow}}^t)\big) - \betacost (\costcloud + \costrobot) &> \\ -\alphaaccuracy \expec\big(\mathcal{L}(\yoracle^t, y_{\text{fast}}^t)\big) - \betacost \costrobot
    \end{align*}
    \end{small}
Simplifying, we arrive at the desired result:
\begin{small}
\begin{align}
    \pioffload^{*}(s^t) = \indicator 
    \bigg (\frac{\betacost}{\alphaaccuracy} \underbrace{\costcloud}_{\mathrm{extra~compute~cost}} &< \\ \nonumber \expec\big( \underbrace{\mathcal{L}(\yoracle^t, y_{\text{fast}}^t) - \mathcal{L}(\yoracle^t, y_{\text{slow}}^t)}_{\mathrm{task~accuracy~benefit}} \bigg).
\end{align}
\end{small}
\label{proof:general_inference_offload}
\end{proof}

Theorem \ref{thm:gen_offload_oracle} suggests a simple model selection policy, which estimates the model accuracy gap $\DeltaL =  \expec\big( \mathcal{L}(\yoracle^t, y^t_{\text{fast}}) - \mathcal{L}(\yoracle^t, y^t_{\text{slow}}) \big)$, and only chooses the slow model if the gap is greater than a threshold that depends on the relative compute costs and weights of accuracy via $\alpha, \beta$. 
However, the key challenge is that calculating $\DeltaL$ requires querying the slow model and knowledge of the ground-truth value $\yoracle^t$. We now transition to two practical approaches to directly instantiate the guarantees from Theorem \ref{thm:gen_offload_oracle} in practice.

First, we note that in many practical deployment scenarios, the ground-truth oracle values are not present. In such practical settings, the more accurate slow model simply serves as the ground-truth, such as when a slow human supervisor makes ground-truth decisions. In the absence of human annotations, a large, compute-intensive DNN can serve as the slow model and ground-truth.
Thus, we present the following lemma of Theorem \ref{thm:gen_offload_oracle}. 
\begin{lem}
    The optimal model selection policy that solves Problem \ref{prob:inference}, when $y^t_{\text{oracle}}=y^t_{\text{slow}}$ at all times $t$ is:  
\begin{align*}
    \pioffload^{*}(x^t,y^t_{\text{fast}}) = \indicator 
    \bigg (\frac{\betacost}{\alphaaccuracy} \costcloud <  \expec\left[ \mathcal{L}(y^t_{\text{fast}},y^t_{\text{slow}}) \right] \bigg)
\end{align*}
\label{lem:general_inference_offload}
\end{lem}
\begin{proof}
    The proof is the same as Theorem \ref{thm:gen_offload_oracle}, where we note that $\expec \big(\mathcal{L}(y^t_{\text{oracle}},y^t_{\text{slow}}) \big) = 0$ when the oracle and slow models are identical.
\end{proof}
For our evaluation, we use Lemma \ref{lem:general_inference_offload} as our model selection policy as it best reflects practical autonomous deployments. The key challenge to directly applying Theorem 1 and Lemma 1 is to accurately estimate the loss between fast and slow models solely using predictions from the fast model. However, we now show we can indeed compute the expected accuracy benefit
for a broad class of fast and slow models that are related by provable approximation guarantees.
Specifically, in Subsection \ref{subsec:linear_reg}, we instantiate the guarantees of Lemma \ref{lem:general_inference_offload} to provide a closed-form, analytic model selection policy for linear regression problems. Crucially, we then extend our analysis to DNN inference in Subsection \ref{subsec:dnn}.
Lemma \ref{lem:general_inference_offload} provides a general framework for model selection. For a novel setting, it can be instantiated by selecting the: (i) compute models, (ii) loss function, (iii) compute model cost, and (iv) characterizing the statistical relationship between compute models to derive the selection policy. Sections \ref{subsec:linear_reg} and \ref{subsec:dnn} instantiate Lemma \ref{lem:general_inference_offload} for specific cases of Linear Regression and DNN inference.

\subsection{Analytical Results for Linear Regression}
\label{subsec:linear_reg}
We now apply the guarantees from Lemma \ref{lem:general_inference_offload} to an illustrative warm-up example of high-dimensional linear regression.
Recall that our challenge is to estimate the expected value of $y^t_{\text{slow}}$ from the information available to the selection policy, namely input $x^t$ and fast model prediction $y^t_{\text{fast}}$.
To overcome this challenge, we apply results to approximate linear regression models using \textit{coresets} \cite{Boutsidis2012RichCF}, which are importance-ranked subsets of a large training dataset. Importantly, a model trained on just the coreset will provably approximate the predictions of one trained on the full dataset. For example, a fast model could be trained on only a core-set of local data on-board a robot while a large one could be trained on multiple robots' data in the cloud.
%

\tu{Compute Models}: Let the fast and slow compute models $f_i$ be linear regression models $f_i(x)=A_ix+b_i$, where $i \in \{\text{fast},\text{slow}\}$, $A_i \in \mathbb{R}^{m \times n}$, and $b_i \in \mathbb{R}^{m \times 1}$. 
We assume the slow model $f_{\text{slow}}$ is learned on a full set of training samples from a joint distribution on $\mathcal{X} \times \mathcal{Y}$, while the fast model is only trained on a core-set of the original data.

\tu{Loss Function}: Let the loss function be the standard $L_2$ norm loss: 
$
\mathcal{L}(y^t_1,y^t_2) = 
||y^t_1 - y^t_2||^2_2
$;
where $y^t_1, y^t_2 \in \mathbb{R}^m$.
The following coreset guarantees follow from \cite{Boutsidis2012RichCF}:
\begin{property}[Relation between fast and slow models \cite{Boutsidis2012RichCF}]
For all $t$, given input $x^t$, denote the compute model outputs as $y^t_{\text{fast}}=f_{\text{fast}}(x^t)$ and $y^t_{\text{slow}}=f_{\text{slow}}(x^t)$. Then, there exists an $\epsilon > 0$ such that:
\begin{align}
\label{eq:coresetLR}
 y^t_{\text{fast}} \in \left[y^t_{\text{slow}}, (1 + \epsilon) y^t_{\text{slow}} \right], 
\end{align}
where $\mathcal{X}$ and $\mathcal{Y}$ are the input and output distributions, meaning $x^t \sim \mathcal{X}$, and $y^t_{\text{slow}}, y^t_{\text{fast}} \sim \mathcal{Y}$. \cite{Boutsidis2012RichCF} provides the approximation factor $\epsilon$ based on the relative size of the core-set compared to the full training set. 
\end{property}

Property 1 allows us to relate the fast and slow model predictions as:
\begin{eqnarray}
    &~& y^t_{\text{slow}} \le y^t_{\text{fast}}
    \le 
    (1+ \epsilon) y^t_{\text{slow}} \nonumber \\
    &or,&
    y^t_{\text{slow}} \in \left[\frac{y^t_{\text{fast}}}{1+\epsilon},y^t_{\text{fast}}\right]
    \label{eq:pacLR}
\end{eqnarray}
Thus, the loss function can be upper bounded as:
\begin{equation}
\label{eq:lossLR}
\mathcal{L}(y^t_{\text{fast}}, y^t_{\text{slow}}) \le
\frac{\epsilon^2 \cdot (y^t_{\text{fast}})^2 }{(1+\epsilon)^2}.
\end{equation}
Finally, we can use Equation \ref{eq:lossLR} and Lemma \ref{lem:general_inference_offload} to provide a closed-form model selection policy for Problem \ref{prob:inference} in the linear regression setting: 
\begin{eqnarray}
\label{eq:offPolLR}
    \pioffload(x^t,y^t_{\text{fast}}) = \indicator 
    \Bigg (\frac{\betacost}{\alphaaccuracy} \costcloud <  \frac{\epsilon^2 \cdot (y^t_{\text{fast}})^2 }{(1+\epsilon)^2} \Bigg).
\end{eqnarray}
We stress that Lemma 1 provides the \textit{optimal} solution and the above solution is an approximation since we upper-bound the loss function between fast and slow models. However, our subsequent experiments show this is a very tight bound and implementing Eq. \ref{eq:offPolLR} yields very close performance to an unrealizable oracle solution that has perfect knowledge of the fast and slow model predictions.

\subsection{Analytical Results for Deep Neural Networks (DNNs)}
\label{subsec:dnn}
We now provide a similar analysis to the linear regression scenario for the important case when a robotic perception DNN has been compressed using recently developed coreset guarantees \cite{baykal2019datadependent,liebenwein2020provable}.
Specifically, \cite{baykal2019datadependent} compresses fully connected DNNs with ReLU activations by targetedly removing weights with low relative importance via coresets. \cite{liebenwein2020provable} extends this work to convolutional neural networks (CNNs) by using coresets to remove convolutional filters that a prediction is least sensitive to, which enables a compressed DNN to provably approximate its original counterpart.

\tu{Compute Models}: Let the models $f_i$: $\mathbb{R}^n \rightarrow \mathbb{R}^m$, where $i \in \{\text{fast},\text{slow}\}$ be DNNs. Both models are trained on a set of samples drawn from a joint data distribution on $\mathcal{X} \times \mathcal{Y}$.

\tu{Loss Function}: As for linear regression, the loss function is an $L_2$ norm loss, such as for depth estimation from a perception CNN. 

We now use the following guarantees for DNNs. 

\begin{property}[Relation between fast and slow models]
For all $t$, given $x^t$, $y^t_{\text{fast}}=f_{\text{fast}}(x^t)$, and $y^t_{\text{slow}}=f_{\text{slow}}(x^t)$, there exist an $\epsilon, \delta > 0$, such that the following holds \cite{baykal2019datadependent,liebenwein2020provable}:
\begin{align}
\mathbb{P} \Big ( y^t_{\text{fast}} \in \big[(1 - \epsilon) y^t_{\text{slow}},(1 + \epsilon) y^t_{\text{slow}}\big] \Big) \ge 1 - \delta, \label{eq:coresetDNN} \\ 
\mathbb{P} \bigg ( y^t_{\text{fast}} \in \left[y^t_{\text{slow}} - \frac{M}{2}, y^t_{\text{slow}} + \frac{M}{2}\right] \bigg) \le \delta,  
\label{eq:gencoreset2}
\end{align}
where $M \ge 0$ is an upper bound on the error described below. $\epsilon$ and $\delta$ depend on the extent of DNN compression.
Further, input $x^t \sim \mathcal{X}$ and outputs $y^t_{\text{slow}}, y^t_{\text{fast}} \sim \mathcal{Y}$.
\end{property}
Equation \ref{eq:gencoreset2} and bound $M$ arise from the observation that in practical engineering scenarios, the outputs of a neural network and thus the loss will be bounded since they have physical meaning. For example, for a regression loss with perception, $M \ge 0$ could be derived from the largest depth-reading a depth sensor can register. Likewise, for classification, $M$ is naturally bounded by 1 since the outputs $y$ are softmax scores from a cross-entropy loss.

We now use the core-set relationship to analyze Lemma 1 for DNN inference as follows:
\begin{equation}
    y^t_{\text{slow}}
     \in \begin{cases} \left[ \frac{y^t_{\text{fast}}}{1+\epsilon}, \frac{y^t_{\text{fast}}}{1-\epsilon}\right] &\mbox{with prob. } \ge 1-\delta \\
    \left[ y^t_{\text{fast}}- \frac{M}{2}, y^t_{\text{fast}} + \frac{M}{2} \right] 
    &\mbox{with prob. } \le \delta. \end{cases}
    \label{eq:pacDNN}
\end{equation}

Thus, using Equation \ref{eq:pacDNN}, the loss can be upper bounded as:
\begin{equation}
\label{eq:lossDNN}
\mathcal{L}(y^t_{\text{fast}}, y^t_{\text{slow}}) \le
\begin{cases}
 \frac{\epsilon^2 \cdot {(y^t_{\text{fast}})}^2}{(1-\epsilon)^2} & \text{with probability}~(1-\delta) \\
M &\text{with probability}~\delta.
\end{cases}
\end{equation}

Therefore, the expectation of the loss function is:
\begin{equation}
\expec\left[ \mathcal{L}(y^t_{\text{fast}},y^t_{\text{slow}}) \right] \le \delta M + (1-\delta) \left( \frac{\epsilon^2 \cdot {(y^t_{\text{fast}})}^2}{(1-\epsilon)^2} \right).
\label{eq:expLossDNN}
\end{equation}

Finally, we can apply Equation \ref{eq:expLossDNN} and Lemma \ref{lem:general_inference_offload} to provide a closed-form model selection policy for Problem \ref{prob:inference} in the DNN setting:
\begin{equation}
\label{eq:offPolDNN}
    \pioffload^{}(x^t,y^t_{\text{fast}}) = \indicator 
    \Bigg (\frac{\betacost}{\alphaaccuracy} \costcloud <
    \delta M + (1-\delta) \frac{\epsilon^2 \cdot {(y^t_{\text{fast}})}^2}{(1-\epsilon)^2} \Bigg)
\end{equation}

As for linear regression, we emphasize Lemma 1 is optimal and Eq. \ref{eq:offPolDNN} is an approximation since we are \textit{bounding} the expectation using the core-set guarantee. However, our experiments show that implementing Eq. \ref{eq:offPolDNN} as a proxy for Lemma 1 works well in practice. More broadly, we emphasize that core-set guarantees are simply one way to instantiate the general policy provided in Lemma 1. For example, a roboticist could also use other practically-relevant models such as random forests or even approximate databases if they can reliably relate fast and slow model accuracy.

%
%
%

\section{Application Scenarios}
\label{sec:app_scenarios}
We now describe example application scenarios of high-dimensional linear regression, DNN inference, and reachable set computation for a simulated Mars Rover to demonstrate the theoretical guarantees from Section \ref{sec:modelSelection}.

\subsection{Linear Regression}
\label{subsec:app_linR}
Using the analytical results from Subsection \ref{subsec:linear_reg}, we demonstrate our model selection policy by simulating it on a toy example of linear regression. Let $f_{\text{slow}}: \mathbb{R}^4 \rightarrow \mathbb{R}^4_{[0,1]}$ be any general-purpose linear regression model. The amount of time $f_{\text{slow}}$ takes to generate an output is 2.5 seconds. Using coresets, we compress the linear regression model $f_{\text{slow}}$, to a faster linear regression model $f_{\text{fast}}$. The $f_{\text{slow}}$ model takes 1 second to generate an output. The compression is such that the relation in Equation \ref{eq:coresetLR} holds with $\epsilon=0.1$.

\SC{We chose $\alpha=1$ and $\beta=0.003$ to emphasize accuracy over compute efficiency in our simulations, although $\alpha, \beta$ can be flexibly set by a user.} We implement the model selection policy as in Equation \ref{eq:offPolLR}, and demonstrate its performance in Section \ref{sec:eval} against benchmark policy selection algorithms (discussed in Subsection \ref{subsec:benchAlgo}).

\subsection{Compute Efficient Robotic Perception}
\label{subsec:app_dnn}

We now stress-test our algorithm on a scenario, inspired by \cite{taxinet}, where an aircraft must autonomously track a runway center-line using a wing-mounted camera for state estimation. This scenario, henceforth referred to as the TaxiNet scenario as per \cite{taxinet}, uses a DNN to map from camera images to an estimate of the aircraft's lateral distance from the runway center-line $d$ and heading angle $\theta$, which are linearly combined to create the aicraft's steering control. We chose the TaxiNet scenario since the central idea is broadly applicable to resource-constrained robotics, such as low-power drones that use efficient vision models to estimate their real-time pose relative to a landing site.


We trained a ResNet-18 DNN \cite{resnet18} to serve as the slow perception model $f_{\text{slow}}$ using over 50K images from the standard X-Plane simulator \cite{XPlane} using a publicly-available dataset \cite{taxinetDataset}. The ResNet-18 achieved a low MSE loss of 0.038 on an independent test dataset of 18,372 images, where each image took 0.17 seconds for inference on a CPU. We compressed $f_{\text{slow}}$ to yield a quantized ResNet-18 as $f_{\text{fast}}$, which was 47.21\% faster but had a 64\% higher loss, illustrating a clear need for model selection.
    
For the TaxiNet scenario, we chose the model costs of $c_{\text{slow}}=0.017$ and $c_{\text{fast}}=0$ based on their relative inference times and $\alpha=1$, $\beta=3 \times 10^{-4}$ to emphasize safety (low loss) over compute costs. Our model selection results are demonstrated in Section \ref{sec:eval} along with example DNN predictions from aircraft images in Figure \ref{fig:all_qual} (Right).

\subsection{Reachable Set Computation}
\label{subsec:reachSet}

In this subsection, we apply our model selection policy to safety assessment for robot navigation.
Consider a robot, such as a Mars rover, navigating an unexplored environment.
The robot has to assess whether its maneuvers are safe while considering environment uncertainties such as the coefficient of friction, wind disturbances, etc..
This is done by computing a \emph{reachable set}, a set that contains all the states a rover can potentially reach.
%
%
The robot can make the maneuver safely if the reachable set does not overlap with any obstacles.
The reachable set computed by the fast compute model has confidence that is an order of magnitude lesser than the reachable set computed by the slow model.

We assume that the closed loop dynamics of the robot is given as a nonlinear system.
For computing the reachable sets, we approximate the nonlinear dynamics locally as an uncertain linear system, where the coefficients in the dynamics belong to a bounded range. 
\begin{example}
\label{eg:uncetainDyn}
Consider the discrete uncertain linear dynamical system $x^{+} = \Lambda x$ where $\Lambda$ = 
$
\begin{bmatrix}
    1  & \alpha \\
    4   & 6, 
\end{bmatrix}
$
where $x$ is the state, $x^{+}$ is the next state, and $\alpha \in [-2,2]$ represents either the modeling uncertainty or a parameter. 
%
Given an initial state $x$, the reachable set of the uncertain linear system includes the set of states reached by the system for any value of $\alpha$ in the interval $[-2,2]$ for a specified time horizon $t$.
\end{example}


%


Prior work~\cite{inproceedings,7318279,10.1145/3358229} has shown that computing reachable sets for linear systems with uncertainties is a computationally expensive process.
Recently, a statistical approximation of the reachable set has been presented in~\cite{ghosh2020}.
The confidence of the statistical approximation can be tuned by the user according to her performance and accuracy requirements.
Leveraging the flexibility of this statistical approach, we generate a fast compute model which has medium confidence and slow model that has high confidence over the computed reachable sets.
Given the various constraints on robot resources, the model selection policy should invoke the appropriate compute model to guarantee safety while minimizing the cost.

Formally, consider a linear dynamical system with uncertainties, represented as $x^+=\Lambda x$, where $\Lambda \subset \reals^{n \times n}$ is an uncertain dynamical matrix. 
The reachable set of the current state $x$ up to a time horizon $t$, is denoted as $\texttt{RS}(\Lambda,x,t)$. 
Though the dynamics is given as $x^+ = \Lambda x$, it can encompass the open loop behavior $x^+ = \Lambda_A x + \Lambda_B u$ (where $\Lambda_A,\Lambda_B$ are uncertain matrices), if a control sequence $u$ is provided. 
In such cases, the uncertain matrices $\Lambda_A,\Lambda_B$ are combined together to $\Lambda$ by concatenating the state and open-loop control.

A system is unsafe if the reachable set intersects with the unsafe set, such as obstacles. That is, given an unsafe set $U \subset \mathbb{R}^n$, a system is unsafe if and only if $\texttt{RS}(\Lambda,x,t) \cap U \ne \emptyset$.
%
%
Given $\mu > 0$ and a set $S \subseteq \mathbb{R}^n$, we denote the uniform expansion (\textit{bloating}) of set $S$ by $\mu$ as $B_{\mu}(S)$.

We now formally present the model selection problem for safety assessment of robot navigation.


\tu{Computation Models}: The compute models $i \in \{\mathrm{fast}, \mathrm{slow}\}$, denoted as  $\texttt{RS}_{i}(\Lambda,x,t)$, compute approximations of the reachable set of an uncertain linear system defined by $\Lambda$ \cite{ghosh2020}. The statistical guarantee $\mathcal{G}_i$ associated with $\texttt{RS}_{i}$ is as follows:
$$
\mathcal{G}_i:~\text{for any}~A \in \Lambda,~ \mathbb{P} \big(\texttt{RS}_{i}(A,x,t) \subseteq \texttt{RS}_{i}(\Lambda,x,t)\big) \ge p_i
$$
$\mathcal{G}_i$ has a \emph{type I error} of $\delta_i$. Here, the confidence $p_i \in \mathbb{R}_{[0,1]}$ and allowable type I error $\delta_i$ are user-given parameters to the models. 
Intuitively, $\mathcal{G}_i$ means the probability that the reachable set of any sample dynamics is contained within the reachable set $\texttt{RS}_{i}(\cdot)$ is at least probability $p_i$. 
Computing high-confidence approximations of the reachable set requires more statistical samples and therefore a higher computational time and cost. 
In particular, the required confidence $p_i$ set by a user for statistical guarantee $\mathcal{G}_i$ is directly proportional to the required number of samples. 
Thus, we set the slow model to be a high-confidence reachable set and the fast model to be a lower-confidence approximation, so $p_{\mathrm{slow}} > p_{\mathrm{fast}}$, $\delta_{\mathrm{slow}} \le \delta_{\mathrm{fast}}$, and therefore $c_{\mathrm{slow}} > c_{\mathrm{fast}}$.
We denote the outputs of the fast and slow models 
as $y^t_{\text{fast}}=\texttt{RS}_{\text{fast}}(\Lambda,x,t)$ and $y^t_{\text{slow}}=\texttt{RS}_{\text{slow}}(\Lambda,x,t)$.


The crux of our selection policy is that we can relate the reachable sets returned by both models by a factor of $\epsilon$:
\begin{property}[Relationship between fast and slow models]
Given $\Lambda$ and $\theta$, for all $t$, there exists an $\epsilon$ such that:
\begin{eqnarray}
 &~&
 B_{1-\epsilon} \big(\texttt{RS}_{slow}(\Lambda,x,t) \big) \subseteq  \texttt{RS}_{fast}(\Lambda,x,t)  \\
 &\text{or,}&
 \texttt{RS}_{\mathrm{slow}}(\Lambda,x,t) \subseteq B_{\frac{1}{1-\epsilon}} \big( \texttt{RS}_{\mathrm{fast}}(\Lambda,x,t) \big).
\label{eq:pacRS}
\end{eqnarray}
\end{property}
In a calibration dataset, we can compute the fast and slow model reachable sets for all time steps. Then, we can set $\epsilon$ to be the minimum factor to bloat the robot's set such that the bloated version over-approximates the slow model's reachable set at all times. 
%
Thus, a robot can quickly run the fast model, bloat it by $\frac{1}{1-\epsilon}$, and continue planning if the \textit{bloated} set does not intersect an unsafe region, as formalized below. 

\tu{Loss Function}: Given the safety-critical nature of navigation, the loss is 0 when the reachable set doesn't intersect an unsafe set and $\infty$ otherwise. Defining the reachable sets used to compute intersections with obstacles as $\Theta_{\text{fast}} =
B_{\frac{1}{1-\epsilon}} (y^t_{\mathrm{fast}})$ and $\Theta_{\text{slow}} = y^t_{\mathrm{slow}} = \texttt{RS}_{\mathrm{slow}}(\Lambda,x,t)$, the loss for any model $i \in \{\mathrm{fast}, \mathrm{slow}\}$ is: 
\begin{equation}
\begin{small}
\mathcal{L}(\Theta_i) 
=
\begin{cases}
    0 & \Theta_i \cap U  = \emptyset
    \\
    \infty & \mathrm{otherwise}.
\end{cases}
\end{small}
\label{eq:loss_nav}
\end{equation}


Finally, using Equation \ref{eq:pacRS} and Theorem \ref{lem:general_inference_offload}, the model selection policy that solves Problem \ref{prob:inference} for safety assessment during robot navigation is:
\begin{eqnarray}
\label{eq:offPolRS}
    \pioffload^{*}(y^t_{\text{fast}}) = \indicator 
    \Bigg (B_{\frac{1}{1-\epsilon}}(y^t_{\text{fast}}) \cap U \ne \emptyset \Bigg).
\end{eqnarray}
Intuitively, the above policy exploits the relationship between fast and slow models by first bloating the fast model's reachable set by a factor of $\frac{1}{1-\epsilon}$ to create a guaranteed over-approximation of the slow model's reachability computation. 
%
If the over-approximation does not intersect obstacles, we are guaranteed safety and simply proceed. If not, we need to invoke the slow model to assess its higher-fidelity reachable set and re-plan a trajectory if it indicates unsafety.
While we implemented our policy with $\alpha=0.7,\beta=0.3$ to prioritize safety, safety is also heavily emphasized in the loss function (Eq. \ref{eq:loss_nav}) since the penalty is $\infty$ for collisions.


\subsection{Benchmark Algorithm}
\label{subsec:benchAlgo}
We evaluate the performance of our model selection policy against the following benchmark policies:

    \noindent \tnm{Fast}: This policy always uses the fast model with prediction $y^t_{\text{fast}}$ for all $t$.

    \noindent \tnm{Slow}: This policy always uses the slow model with prediction $y^t_{\text{slow}}$ for all $t$.

    \noindent \tnm{Random}: The robot randomly chooses between the fast and slow model with equal probability.

    \noindent \tnm{Our Selector}: This represents our model selection policy from Equations \ref{eq:offPolLR}, \ref{eq:offPolDNN}, and \ref{eq:offPolRS}.

    \noindent \tnm{Oracle}: This strategy assumes that the slow model's output is available to the model selection function at the time of inference. Thus, this strategy only selects the slow model when that decision has a better reward than using the fast model. The oracle is an upper-bound, unrealizable strategy since it assumes privileged knowledge of the slow model.




\section{Evaluation}
\label{sec:eval}
The principal objective of our evaluation is to show that our model selection policies
from Lemma 1 and Equations \ref{eq:offPolLR}, \ref{eq:offPolDNN}, and \ref{eq:offPolRS}
achieve a significantly higher reward than benchmark model selection policies.
Further, we show how our policy achieves better accuracy with a lower cost than competing
benchmarks on simulations of linear regression, aircraft taxiing with state-of-the-art DNN perception models, and rover navigation with real Martian terrain data. All our code (in Python) and models are publicly available at \cite{modelselection}.

\begin{figure*}[ht]
  \centering
 \includegraphics[width=15cm,height=3.8cm]{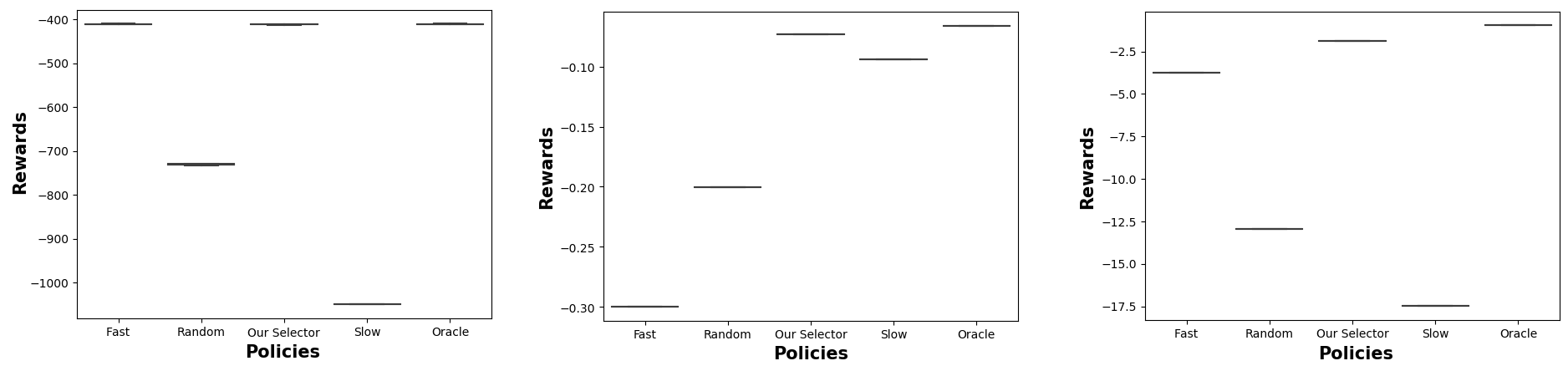}
 \caption{\textbf{Rewards. (\textit{Left: Linear Regression, Center: DNN, Right: Mars Rover Reachable Set})}. We illustrate the cumulative rewards (Eq. \ref{eq:inference_reward}) gathered by various policies on the Linear Regression, DNN perception (TaxiNet), and Mars Reachable Set scenarios, respectively. Clearly, our policy (\tnm{Our Selector}) achieves the maximum reward compared to other realizable benchmarks and is close to the oracle in all cases.}
 \label{fig:all_rwd}
\end{figure*}

\begin{figure*}[ht]
  \centering
 \includegraphics[width=16cm,height=4cm]{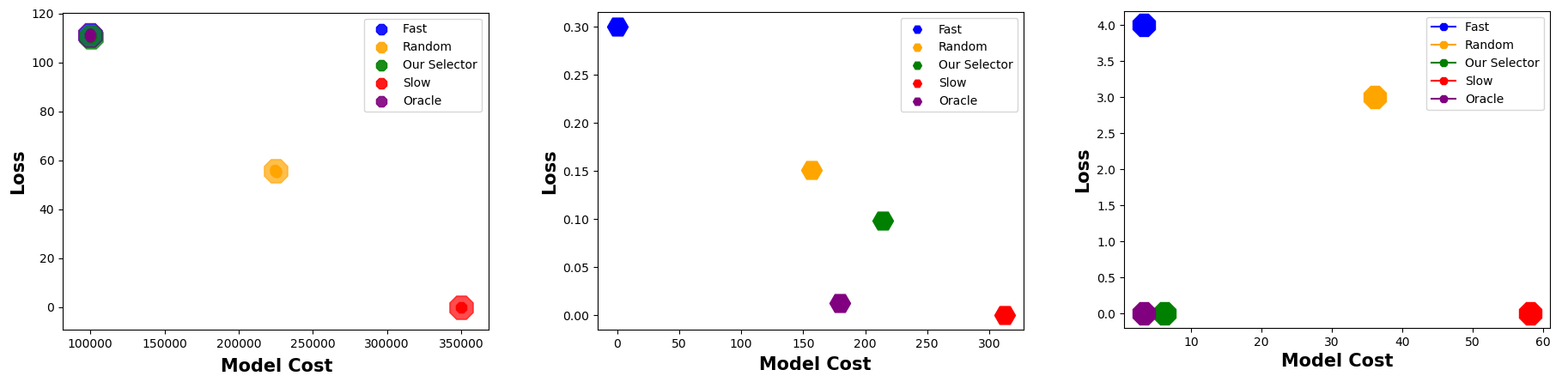}
 \caption{\textbf{DNN Cost vs. Accuracy. (\textit{Left: Linear Regression, Center: DNN, Right: Mars Rover Reachable Set})}. Cost vs. Loss trade-off achieved by various model selection policies on all scenarios. In all cases, we observe that the \tnm{Fast} policy has low cost but low accuracy, \tnm{Slow} has high accuracy but high cost, and \tnm{Random} lies sub-optimally in the middle (with a high variance). Only the selection policy proposed in this paper (\tnm{Our Selector}) achieves a delicate balance by exploiting the statistical relationship between models to intelligently consult the slow model.}
 \label{fig:all_cvl}
\end{figure*}

\begin{figure*}[ht]
  \centering
 \includegraphics[width=16.2cm,height=3.2cm]{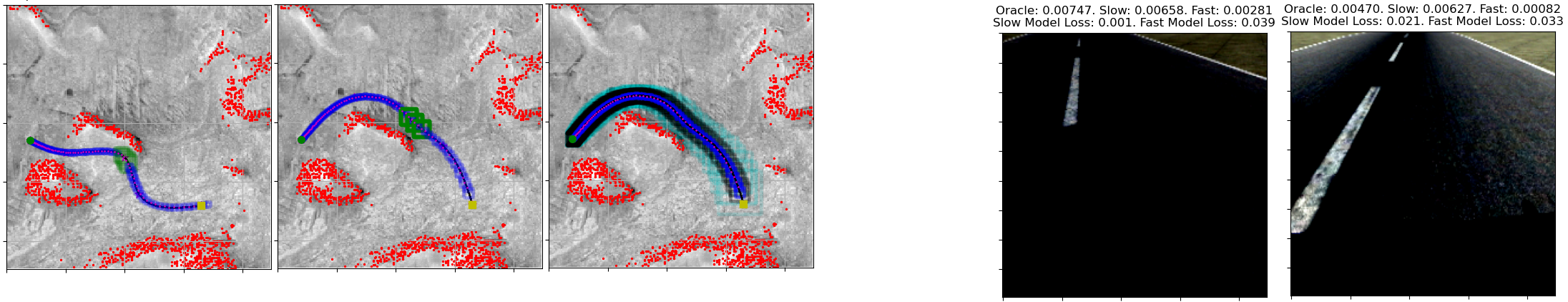}
 \caption{\textbf{Left: (Safe  Navigation  of  a  Mars  Rover).} We consider navigation of a simulated rover on real Mars HiRise terrain data \cite{hirise}, where the red point clouds are obstacles indicating regions of high elevation, as processed in \cite{nakanoya2020taskrelevant}. \textit{First two images:} We show the reachable sets of two different possible routes taken by the Mars Rover. The reachable set in green is computed by the slow model, whereas the one in blue comes from the fast model. 
 Path safety is determined by assessing if the reachable set (accounting for uncertainties), intersects red obstacles. Clearly, our model selection policy uses the slow model only when the rover makes tricky maneuvers. Otherwise, when the rover is far from obstacles, the fast model is sufficient to determine safety, allowing us to maintain high safety with a lower compute time. 
  \textit{Third image:} We show the relationship beween reachable sets computed by the fast model (in blue) and the slow model (in black). Crucially, our policy uses the over-approximated reachable set of the slow model as computed by the fast model (in cyan), that is, $B_{\frac{1}{1-\epsilon}}(y^t_\text{fast})$. Clearly, our model selection policy only consults the slow model when it suspects a possible collision when the set represented in cyan intersects with a red obstacle. Our policy always led to safe, collision-free, efficient navigation by exploiting the relationship between fast and slow models. \textbf{Right: (Aircraft TaxiNet DNN Output)}. The output of the fast and slow DNN models for a ResNet-18 TaxiNet model. Given an image, the final output shown is the rudder control.} 
 \label{fig:all_qual}
\end{figure*}


\subsection{Linear Regression Results}
\label{subsec:eval_LR}
We now evaluate our selection policy for linear regression, as described in Equation \ref{eq:offPolLR} and Subsection \ref{subsec:app_linR}. The key highlight is that our policy achieves 245.4\% higher reward than benchmarks in 100 trials, each of duration $N=10^5$ timesteps with stochastic Gaussian inputs $x^t$. Figures \ref{fig:all_rwd} (Left) and \ref{fig:all_cvl} (Left) show the cumulative rewards and trade-off between accuracy and cost, respectively, of all algorithms.



\subsection{Deep Neural Networks (DNN)}
\label{subsec:eval_DNN}

We now evaluate our model selection policy for the TaxiNet aircraft taxiing scenario from Subsection \ref{subsec:app_dnn}. Our key result on 18,372 \textit{test} images is shown in Figure \ref{fig:all_rwd} (Center), where our policy (\tnm{Our Selector}) achieves 22.22\% higher reward than competing benchmarks and is within 10.18\% the performance of an upper-bound \tnm{Oracle}. Moreover, Figure \ref{fig:all_cvl} (Center) shows that our model selection policy achieves low loss with low cost unlike competing policies. This is because our policy leverages the statistical correlation between models to mostly rely on the fast model to reduce cost, but also opportunistically queries the slow model for higher accuracy. However, our policy is careful to only invoke the slow, accurate model when there is a substantial accuracy gain, leading it to be queried only 68.6\% of the time.

\subsection{Reachable Set Computation}
\label{subsec:eval_RS}
We now demonstrate the performance of our model selection policy (Equation \ref{eq:offPolRS}) to determine the safety of a simulated Mars Rover navigating steep obstacles on terrain from NASA's HiRise Dataset \cite{hirise,nakanoya2020taskrelevant}. A low-power rover must always be safe, but also fast and compute-and-power-efficient while accounting for reachable sets while planning. 

The rover is assumed to follow a linearized bicycle model with bounded perturbations in the dynamics matrix for yaw angle. Given an intended path, we use our model selection policy (Equation \ref{eq:offPolRS}) to determine safety given uncertain dynamics while minimizing compute time.
Specifically, given a start set, desired goal, and a set of way-points, we compute a reference trajectory using a cubic spline planner, which is followed using Model Predictive Control (MPC).
Using the planned states $x$ and controls $u$ at every time, our model selection policy must determine the trajectory's safety by invoking either a fast or slow reachable set computation model as described in Subsection \ref{subsec:reachSet}.

Figure \ref{fig:all_qual} (Left, first two images) shows how our policy (Equation \ref{eq:offPolRS}) safely, but efficiently, follows two different paths near a red obstacle indicating an unsafe terrain gradient above 20 degrees. The key benefit of our approach is that the robot mostly uses the fast reachable set computation (blue) for high-efficiency and only \textit{intelligently} consults the higher-fidelity slower model during tricky turns close to an obstacle. 
Indeed, Figure \ref{fig:all_qual} (Left, third image) precisely shows how our policy (Equation \ref{eq:offPolRS}) exploits the relationship between fast and slow models to selectively query the slow model only when required during key turns. The fast model's reachable set result is in blue, the slow model's result is in black, and the \textit{over-approximation} from bloating the fast model's result by $B_{\frac{1}{1-\epsilon}}(y^t_{\text{fast}})$ is in cyan. 

Clearly, even the over-approximation rarely intersects unsafe obstacles and it is only necessary to consult a fine-grained result from the slow model (black) when the over-approximation is too conservative and needs to be refined. In all scenarios, we rigorously verified the simulated rover is safe and never hits an obstacle despite dynamics uncertainties.
Figures \ref{fig:all_rwd} (Right) and \ref{fig:all_cvl} (Right) quantitatively illustrates the superior efficiency and accuracy (safety) of our policy, since it achieves the highest reward, never hits an obstacle, and efficiently only queries the slow model on-demand near critical obstacles. 

\textit{Limitations of Our Work: }
In the future, we plan to account for more sophisticated nonlinear dynamics using Hamilton-Jacobi-Bellman reachability analysis. 
Finally, future work should address multi-step decision-making, where model selection decisions affect subsequent measurements and control decisions.



\section{Conclusion}
\label{sec:conclusion}
To scale the deployment of low-power robotic swarms, it is increasingly important to optimize for compute energy, cost, and latency alongside standard metrics of task accuracy and resiliency. This paper presents a general algorithm for robots to flexibly trade-off task accuracy and compute cost in an interpretable manner with provable statistical guarantees. Our key insight is to leverage the statistical correlations between models to \textit{predict} the marginal accuracy gain of a large model and balance it with additional compute costs. This general principle allows our framework to widely apply to cloud robotics, DNN perception, and reachability analysis.

In the future, we plan to address safety guarantees and investigate whether we can co-train large and small DNNs such that we can synthesize an interpretable run-time monitor that can transfer authority to a trusted controller if the DNNs are operating in uncertain regimes.
Overall, we anticipate our model-selection results will become stronger with future advances in DNN verification and compression with approximation guarantees.



\bibliographystyle{IEEEtran}
\bibliography{main}


\end{document}